\documentclass[letterpaper, 10 pt, conference]{ieeeconf}  

\IEEEoverridecommandlockouts                              
\usepackage{amsmath} 
\usepackage{amssymb}  

\usepackage{amsthm}
\usepackage{mathrsfs}

\PassOptionsToPackage{hyphens}{url}
\usepackage{hyperref}
\hypersetup{breaklinks=true}
\hypersetup{
    colorlinks = true,
    linkcolor = blue,
    urlcolor = blue,
    citecolor = magenta,
}

\usepackage{comment}
\usepackage{cite}

\usepackage{mathtools}

\theoremstyle{definition} 
\newtheorem{defn}{Definition}[section]
\newtheorem{rmk}[defn]{Remark}
\newtheorem{ex}[defn]{Example}

\theoremstyle{plain}
\newtheorem{thm}[defn]{Theorem}
\newtheorem{lem}[defn]{Lemma}

\usepackage{tikz-cd} 

\usepackage[normalem]{ulem}

\newcommand{\define}[1]{\textbf{#1}}

\newcommand{\R}{\mathbb{R}}
\newcommand{\Z}{\mathbb{Z}}

\newcommand{\calA}{\mathcal{A}}

\newcommand{\calS}{\mathcal{S}}

\newcommand{\True}{\top}
\newcommand{\Until}{\mathcal{U}}

\newcommand{\notltl}{\neg}

\newcommand{\Event}{\Diamond}
\newcommand{\Always}{\Box}

\title{\LARGE \bf
Stratifying Reinforcement Learning with Signal Temporal Logic
}

\author{Justin Curry and Alberto Speranzon
\thanks{J.~Curry is with the University at Albany, State University of New York, Department of Mathematics and Statistics, 1400 Washington Ave., Albany, NY 12202 USA.{Email: \url{ jmcurry@albany.edu}}}%
\thanks{A.~Speranzon is with Lockheed Martin, Advanced Technology Labs, 1303 Corporate Center Dr, Eagan, MN 55121 USA. {Email: \url{alberto.speranzon@lmco.com}}.}%
}

\begin{document}

\maketitle

\begin{abstract}
In this paper, we develop a stratification‑based semantics for Signal Temporal Logic (STL) in which each atomic predicate is interpreted as a membership test in a stratified space. This perspective reveals a novel correspondence principle between stratification theory and STL, showing that most STL formulas can be viewed as inducing a stratification of space‑time. The significance of this interpretation is twofold. First, it offers a fresh theoretical framework for analyzing the structure of the embedding space generated by deep reinforcement learning (DRL) and relates it to the geometry of the ambient decision space. Second, it provides a principled framework that both enables the reuse of existing high‑dimensional analysis tools and motivates the creation of novel computational techniques. To ground the theory, we (1) illustrate the role of stratification theory in Minigrid games and (2) apply numerical techniques to the latent embeddings of a DRL agent playing such a game where the robustness of STL formulas is used as the reward. In the process, we propose computationally efficient signatures that, based on preliminary evidence, appear promising for uncovering the stratification structure of such embedding spaces.
\end{abstract}

\section{Introduction}

Recent papers in machine learning~\cite{marchetti2025algebra,robinson2025token} have shown that the ``manifold hypothesis'' does not hold for many neural network embeddings; instead, these  spaces are \emph{stratified}, informally, they consist of manifolds of varying dimensions.
In this work we argue that such stratifications are inherited from (i) the stratified structure of the data (ambient) domain and (ii) the loss (or reward) function used during training. We focus on decision problems, specifically Minigrid games, where the reward is the robustness of a Signal Temporal Logic (STL) predicate and a deep reinforcement learning (DRL) agent is trained to learn winning policies.

At a high level, single agent decision tasks generate directed state transition graphs that naturally form posets: the agent must progress toward goal states while respecting temporal, state space, or adversarial constraints. We illustrate the utility of poset‑based stratifications with three increasingly complex game examples.
The second part of the paper links these ideas to control theory and modern neural‑network architectures. We introduce a network that uses STL robustness as a reward, thereby imposing a goal‑oriented stratification on the ambient space. Numerical experiments demonstrate how we can analyze stratified spaces, both ambient and high dimensional embedding spaces, highlighting promising directions and current limitations.

The relationship between local/intrinsic dimension of modern machine‑learning embeddings and stratification theory has recently attracted attention from the topological data analysis community. Prior work has examined this link for deep network architectures~\cite{nurisso2026topology,valeriani2023geometry}, reinforcement learning~\cite{catanzaro2024topological}, image processing~\cite{brown2022verifying}, and language models~\cite{ruppik2025less} (see references therein).  
However, to the best of our knowledge, no study has yet explored how the stratification of the ambient (data) space induces a corresponding stratification in the latent space, nor how reward functions defined as STL‑robustness values shape the decision‑space.  
We believe this paper opens new research avenues that bridge control theory, machine learning, and algebraic topology. 

The paper is organized as follows. Section~\ref{sec:Stratification_Theory_Games} defines poset stratified spaces, presents game‑based examples, and reviews numerical tools for stratification analysis. 
Section~\ref{sec:STL_RL} recasts STL semantics in terms of semi‑algebraic sets and defines predicates via ``rug functions''~\cite{durfee1983}.
STL is then examined through tubular neighborhoods and stratification theory. Section~\ref{sec:Stratification_DRL} reports preliminary numerical results on the stratification of latent spaces produced by a DRL agent. 

\section{Stratification Theory for Minigrid Games}
\label{sec:Stratification_Theory_Games}

\begin{figure}[t!]
  \centering
  \includegraphics[width=0.85\columnwidth]{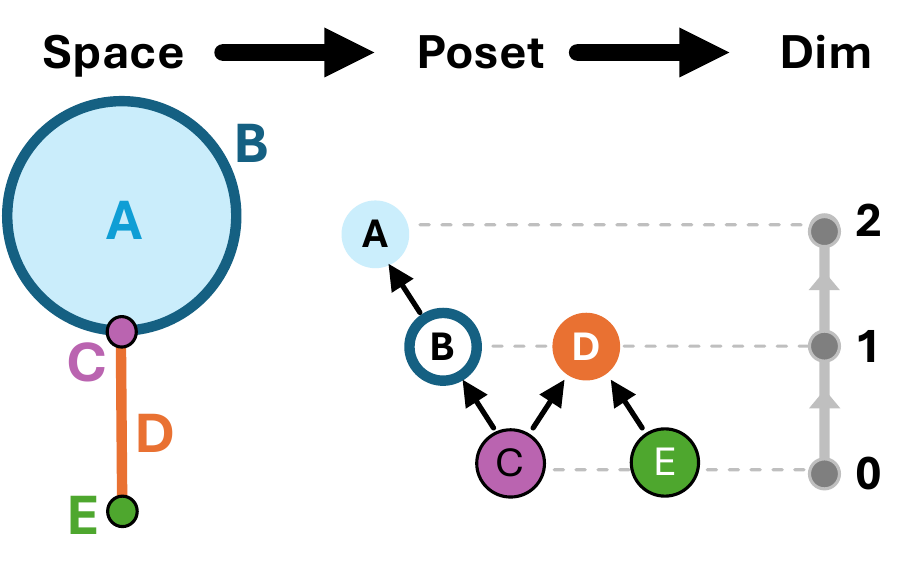}
  \vspace*{-0.20in}
  \caption{A disk with a line attached can be stratified via the face-relation poset, which refines the dimension poset $0 < 1 < 2$.}
  \label{fig:space-poset-dim}
\end{figure}

Historically, stratification theory was developed to describe topological spaces that are ``built out'' of manifolds.
More precisely, the standard definition of a stratified space $X$ specifies a sequence of closed subsets
\[
\varnothing = X_{-1} \subseteq X_{0} \subseteq \cdots \subseteq X_n = X
\]
with the property that each $X^i \coloneqq X_i \setminus X_{i-1}$ is a (possibly empty) $i$-dimensional manifold.
More recently~\cite{waas2024}, mathematicians have adopted the following language:

\begin{defn}
    Let $X$ be a topological space and let $(P,\leq)$ be a partially ordered set (a poset).
    Equip $P$ with the \emph{Alexandrov topology} where down-sets are closed, i.e., whenever $D\subseteq P$ satisfies $y\in D$ and $x\leq y \Rightarrow x\in D$, then we consider $D$ to be a closed set. A \define{poset stratification} of $X$ is then a continuous map $q:X\to P$. If additionally each $X^p \coloneqq q^{-1}(p)$ is a manifold, then we say $q$ is a poset stratification by manifolds.
\end{defn}


The historical definition of a stratified space can be viewed as a special case of a poset-stratified space $q:X \to P$ where $P=\Z$ is the integers and $q$ is the dimension function.
One of the most basic reasons to consider more general posets is to better understand how different parts of a stratified space are pieced together.
For example, Figure~\ref{fig:space-poset-dim} shows how a disk with a line attached can be stratified via the \emph{face-relation poset} $\mathbf{Face}(X)$, which is a poset that records boundary relations.
In Figure~\ref{fig:space-poset-dim} the space $X$ consists of a two-dimensional (open) disk A, whose boundary (or face) is the one-dimensional circular arc B and the point C. The one-dimensional edge D has another boundary point E. 
This poset has an order-preserving map to the dimension poset  $0 < 1 < 2$.

In this paper, we are interested in using the full generality of poset stratifications to model spaces typically outside the realm of traditional topology, focusing on RL games as a special case.


\subsection{Digital Stratifications for Games}
\label{subsec:digital-strats}


For our first class of examples, we consider 2D Minigrid environments \cite{minigrid2023}, which provide popular benchmarks for reinforcement learning (RL) algorithms.
Minigrid games offer mission instructions that instruct an agent, indicated by the red triangle, to navigate to particular regions, like a green square, possibly by first accomplishing certain sub-tasks, like collecting a key and then unlocking a door.

\begin{figure}[t!]
  \centering
  \includegraphics[width=0.8\columnwidth]{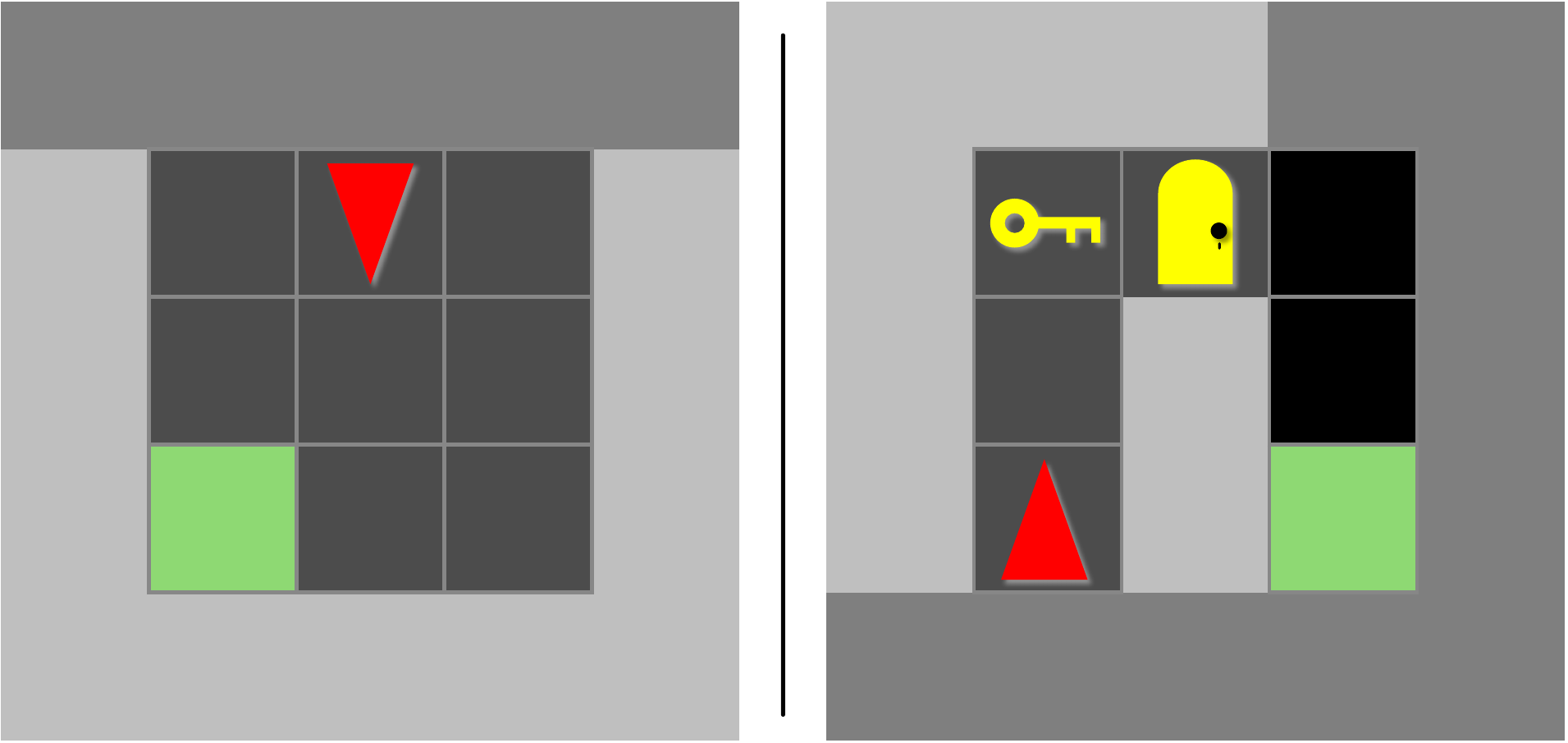}
  \caption{Empty (left) and Door-Key (right) Minigrid Environments \cite{minigrid2023}}
  \label{fig:mini-side}
\end{figure}

In Figure~\ref{fig:mini-side}, two different examples of Minigrid games are shown.
The empty environment (left) where the agent's goal is to navigate to the green square without any obstacles, whereas the door-key example (right) shows a environment where the agent must first collect a key to then open a door, before navigating to the green square.
We can formalize each minigrid game as Markov Decision Process (MDP).
\begin{defn}
    A deterministic \define{Markov Decision Process} (MDP) consists of a set of states $\mathcal{S}$, a set of actions $\mathcal{A}$, a transition function $T:\calS\times \calA \to \calS$, and a reward function $R:\calS\times \calA \times \calS \to \R$. A \define{policy} is an assignment $\Pi: \calS \to \calA$ that specifies what action to take in any given state.
    A policy then specifies a discrete flow $\Phi:\calS\to\calS$ via the composition $\Phi(s) \coloneqq T(s,\Pi(s))$.
    One can then use this flow to specify discrete time trajectories $\gamma_{s}=(s,\Phi(s),\Phi^2(s),\ldots)$
    passing through any given state $s$.
\end{defn}

\begin{ex}\label{ex:rotation-game}
    For the empty minigrid environment (Figure~\ref{fig:mini-side} left), the set of states $\calS$ consists of all $(\text{row},\text{column})$ positions in the $3\times 3$ grid $G$, along with a cardinal direction $C=\{N,E,S,W\}$, indicating the orientation of the red triangle (agent), so $\calS=G\times C$ is a 36-element set. 
    The state of the agent in the left of Figure~\ref{fig:mini-side} is $s=(1,2,S)$ indicating that the agent is in row-1 and column-2 pointing South.
    The set of actions $\calA=\{L,R,F,D\}$ described in \cite{minigrid2023} correspond to turn left, turn right, go forward, and do nothing (or done).
\end{ex}

In Example \ref{ex:rotation-game}, the green square corresponds to any state of the form  $(3,1,\star)$ for $\star\in \{N,E,S,W\}$.
Since all these states can be viewed as equivalent goal states, we consider a simplified version of the minigrid game with no rotations.
    
\begin{ex}\label{ex:no-rotation-game}
    Consider a modified, no rotation, version of the minigrid environment, where $\calS=G$ is the $3 \times 3$ grid and the action space is $\calA=\{Up,Down,Left,Right,None\}$ translates an agent vertically, horizontally, or not at all. A policy $\Pi$ that routes the agent from a state $s\in \calS$ to the goal state $g=(3,1)$ is equivalent to picking a spanning tree indicated in the state graph, an example shown in Figure~\ref{fig:mini-tree}.
\end{ex}


\begin{figure}[t!]
  \centering
  \includegraphics[width=0.8\columnwidth]{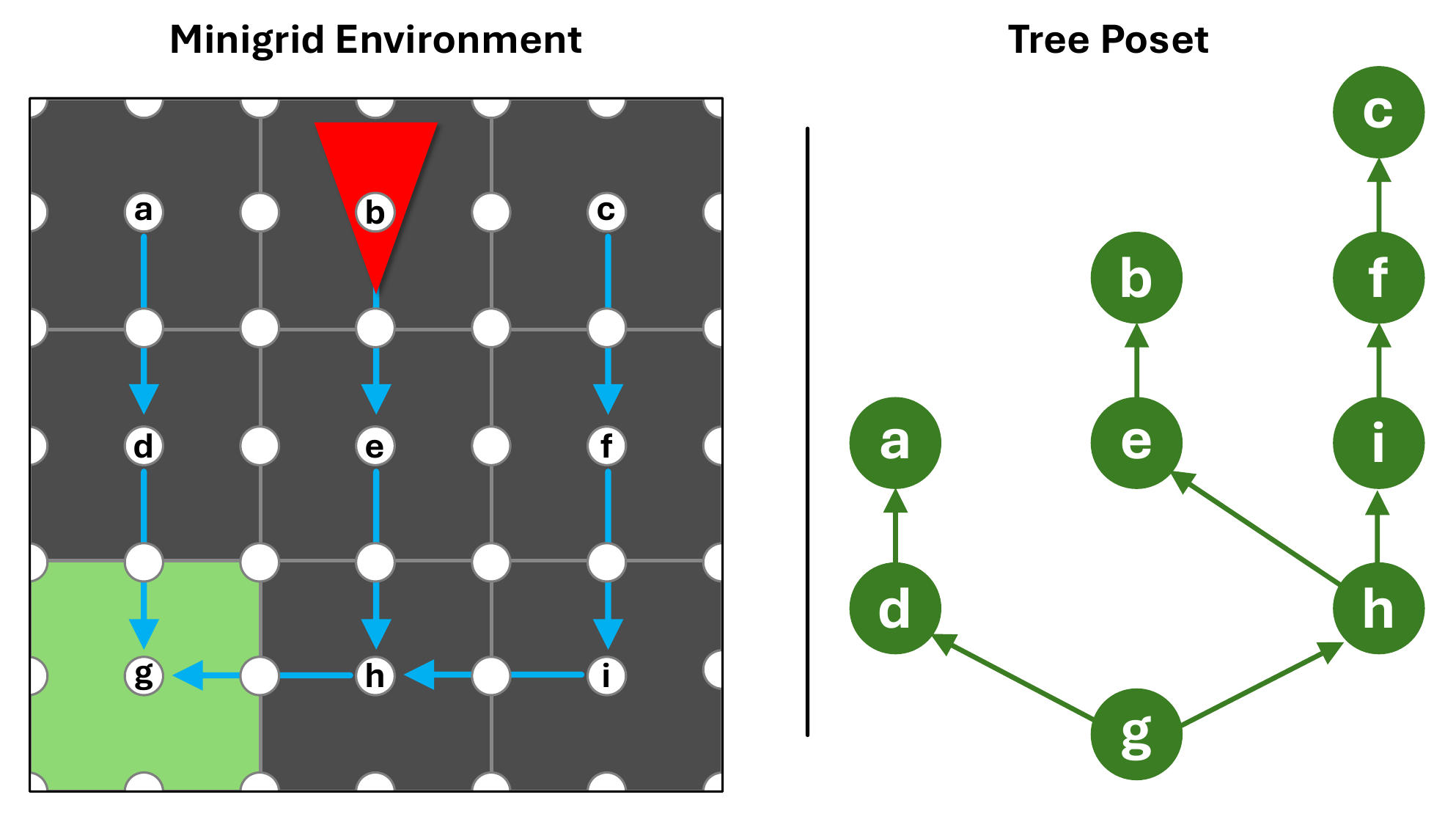}
  \caption{An empty $3\times 3$ Minigrid Environment is first translated into a cell complex with 49 cells (9 two-dimensional cells, labeled \emph{a} through \emph{i}, 24 edges, and 16 vertices), which can be viewed as a finite ``digitized'' topological space with 49 points. Such a topologized minigrid environment has a continuous map to the tree poset displayed right.}
  \label{fig:mini-tree}
\end{figure}

Formally, we can show that minigrid games where we ignore the cardinal direction of the agent can be viewed as a tree-stratified state space, where the tree is akin to the policy tree $T$ for a state graph.
There is one caveat that we must first ``digitize'' our state space to match the implied Euclidean topology on our video game screen, which is explained in the proof of the following lemma.

\begin{lem}\label{lem:tree-strat}
    For any minigrid game where rotations are ignored, as in Example \ref{ex:no-rotation-game}, and with a single goal state $g$, any policy that generates successful trajectories passing through arbitrary start states $s$, can be viewed as defining a poset-stratified space with indexing poset a tree.
\end{lem}
\begin{proof}
    As MDPs do not typically specify a topology we must specify one to define a stratification.
    We accomplish this by viewing our minigrid environment as a cell complex $X$, with a 2-cell (open disk) for each tile in the grid, a 1-cell (an edge) for the boundary between two tiles, and a 0-cell (vertex) for each of the four corners of a pixel. 
    As already explained, this defines a poset $\mathbf{Face}(X)$ where vertices are less than adjacent edges, which are less than adjacent 2-cells.
    See Figure \ref{fig:mini-tree}, where the points of this poset are indicated in white.
    The poset $\mathbf{Face}(X)$ can be topologized so that down-sets are closed, which we regard as the \define{digital} version $X_{dig}$ of the space $X$. Note that there is always a continuous map from a cell complex $X\to X_{dig}$ by collapsing all points in the same cell to a single representative point.
    From $X_{dig}$ we define the map $q:X_{dig}\to T$ to the policy tree by mapping each 2-cell to a unique point in the policy tree.
    For lower dimensional cells, we assign the greatest lower bound (in the tree poset $T$) of the cells above it in the poset $\mathbf{Face}(X)$. This ensures that $q:\mathbf{Face}(X)\to T$ is a poset map, which is automatically continuous with respect to the Alexandrov topology. The stratification of $X$ is then the composition $X\to X_{dig}\to T$.
\end{proof}

\begin{rmk}
    The prohibition on rotation is important when we view MDPs as topological spaces and request continuity of policies. Here, one can think of policies as generalized vector fields.
    We remind the reader that the Poincar\'e-Hopf Index Theorem \cite[Ch. 3, \S 5]{guillemin2025differential} places strong topological limitations on continuous vector fields. 
    If a state space is topologically equivalent to a circle $\mathbb{S}^1$, then any sink (the goal state) must have an accompanying source (unstable fixed point of the policy).
\end{rmk}

\subsection{Stratifying Space-Time ``Coin Collection'' Games}
\label{subsec:coin_strat}

\begin{figure}[t!]
  \centering
  \includegraphics[width=\columnwidth]{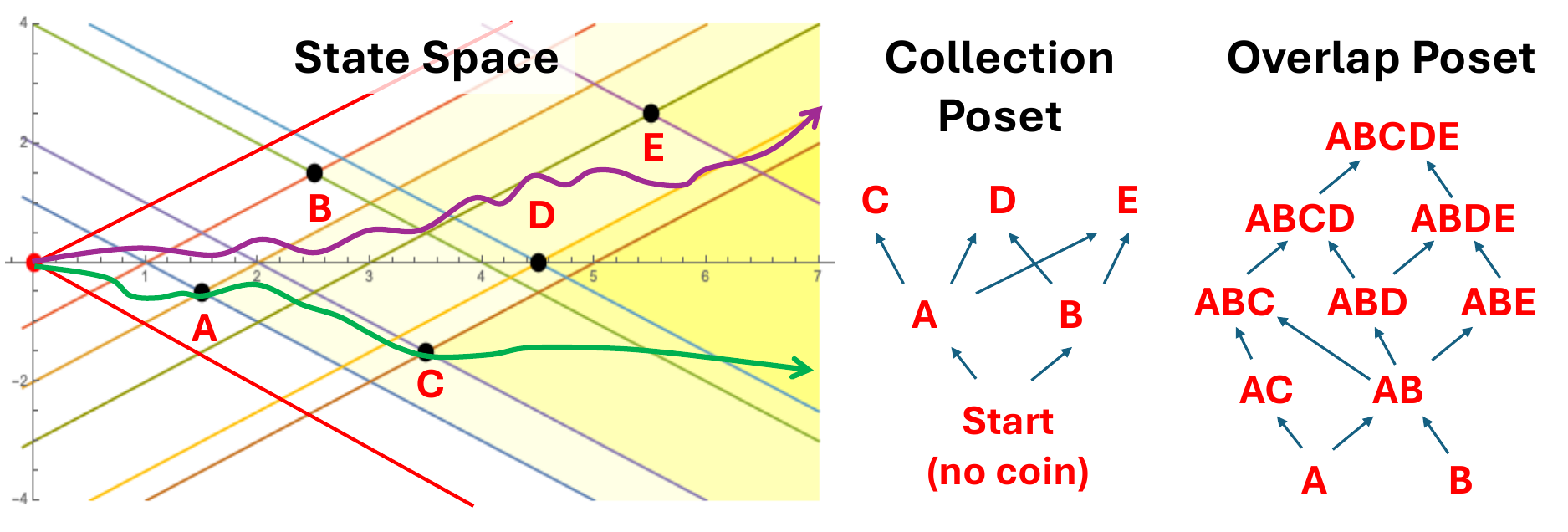}
  \caption{Stratifying a ``Coin Collection'' Game inspired by \cite{baryshnikov2023linear}.}
  \label{fig:coin-strat}
\end{figure}

Moving beyond minigrid games, we turn our attention to a coin collecting game that was inspired by dualizing an obstacle avoidance game considered in \cite{baryshnikov2023linear}.
This game also indicates how space-time considerations can be encoded in an (inherently directed) poset $P$.
In Figure~\ref{fig:coin-strat}, the state space $\mathcal{S}=\{(t,y) \in \R_{\geq 0} \times \R \mid |y|\leq t\}$ $\R_{\geq 0}\times Y$ encodes time as the $x$-coordinate and position as the $y$-coordinate.
The agent's speed is limited in such a way so that $|\dot{y}(t)|\leq 1$ everywhere.
This means that an agent starting at $S=(0,0)$ cannot navigate outside of the red ``light cone'' drawn in Figure~\ref{fig:coin-strat}.
The goal of this game is to locate resources, ``coins,'' that appear at 5 different instants in space-time, labeled $A$ through $E$, without violating the velocity constraints.
After some thought, the reader should convince themselves that at most two coin collections are possible. 
This is, for example, witnessed by the green trajectory in Figure~\ref{fig:coin-strat}, which passes through the coins $A$ and $C$.

As evidenced by the proof of Lemma \ref{lem:tree-strat}, finding a poset stratification is easiest when our state space can be viewed as a poset on its own. Previously, that was accomplished by using a cell structure on the state space, but here $\mathcal{S}$ can naturally be viewed as a poset under the relation $(t,y)\leq (t',y')$ iff $t\leq t'$ and $|y'-y|\leq t'-t$.
Principal up-sets, i.e., sets of the form $U_p:=\{p'\mid p\leq p'\}$ correspond exactly to future light-cones in the poset $\mathcal{S}$.
The \define{collection poset} $\mathcal{C}$, drawn in the middle of Figure~\ref{fig:coin-strat}, is the opposite\footnote{In general the association of a point $p\in P$ with its principal up-set $U_p$ is order-reversing.} of the inclusion poset for the future light cones $U_A, U_B, U_C, U_D, U_E$.
Unfortunately, $\mathcal{S}$ is \emph{not} stratified by the collection poset $\mathcal{C}$, drawn in the middle of Figure~\ref{fig:coin-strat}.
This is because a point right above $D$, say $p=(4.5,1)$ satisfies $A\leq p$ and $B\leq p$, but we cannot choose either of those targets in $\mathcal{C}$ without violating the order-preserving condition.

The answer is to pass to the larger \define{overlap poset} $\mathcal{O}$, which records the maximal depth intersection among the light cones. Indeed if $\sigma\subseteq \{S,A,B,C,D,E\}=:I$ and we write $U_{\sigma}:=\cap_{i\in \sigma} U_i$ for short then
\[
\mathcal{O}:=\{\sigma\subseteq I\mid U_{\sigma} \neq \varnothing \text{ and } \nexists \tau\supset \sigma \text{ with } U_{\sigma}=U_{\tau}\}.
\]
Since every point in $\mathcal{S}$ is in $U_S$, we can map every $p\in \mathcal{S}$ to a point in $\mathcal{O}$ and this map is order-preserving.
Note that we have suppressed that extra letter $S$ from the Hasse diagram for $\mathcal{O}$, drawn right in Figure~\ref{fig:coin-strat}, so $\mathcal{O}$ actually has 11 points, and not just the 10 shown.

\subsection{Stratifying Trajectory Space}
\label{subsec:strat_traj}

From an RL perspective, stratifying the state space is not the only object of interest.
Rather, one wants to stratify the \emph{trajectory space}, especially when the state space does not support an obvious poset structure.
To that end, we will prove that the trajectory space for our coin collection game is actually stratified by the \define{order complex} of the collection poset, written $K(\mathcal{C})$, whose elements are chains of comparable elements in $\mathcal{C}$ ordered by reverse inclusion.
We put the empty chain at the top of $K(\mathcal{C})$, then all of the length 0-chains below that, with the length 2 chains at the bottom.
We do this to emphasize that passing through no distinguished state is the generic behavior, where as passing through multiple distinguished states is harder ``low-dimensional'' behavior.
We make this precise with the following theorem.

\begin{thm}\label{thm:traj-strat}
    Suppose $T\subset X$ is a closed set of ``target values.'' Let $C^0(\mathbb{R},X)$ denote the space of continuous maps from $\R$ to $X$, i.e., trajectories.
    Let $\mathcal{Z}$ be the poset of closed subsets of $\R$, ordered by reverse inclusion.
    Then
    \[
    q_T : C^0(\mathbb{R},X) \to \mathcal{Z}:=\mathbf{Closed}(\R)^{op} \qquad \gamma \mapsto \gamma^{-1}(T)
    \]
    is a poset stratification.
\end{thm}
\begin{proof}
    First the map $q_T$ is well-defined because if $\gamma$ is continuous, $Z:=\gamma^{-t}(T)$ is closed so $q_T(\gamma)=Z\in \mathcal{Z}$ is well-defined.
    To understand why we need to use the opposite containment relation on $\mathbf{Closed}(\R)$, recall that closedness in $C^0(\mathbb{R},X)$ is defined in terms of sequential convergence, i.e, $\{\gamma_n\}\to \gamma$ if for every $t\in \R$ $\gamma_n(t)\to \gamma(t)$.
    Now if $\{\gamma_n(t)\}$ is a sequence of points in our target set $T$, then since $T$ is closed the limit $\gamma(t)\in T$ is closed as well.
    More formally, this observation proves that for any closed set $S\in \mathbf{Closed}(\R)$ the set of functions $F_{S,T}:=\{\gamma\mid \gamma^{-1}(T)\supseteq S\}$ is closed.
    This is precisely inverse image of the down-set $D_S:=\{Z\in \mathbf{Closed}(\R) \mid Z\supseteq S$ under $q_T$, which proves continuity of the stratification map.
\end{proof}

\subsection{Stratification Detection}
\label{subsec:stratification_dectection}


\begin{figure}[t!]
    \centering
    \includegraphics[width=0.75\columnwidth]{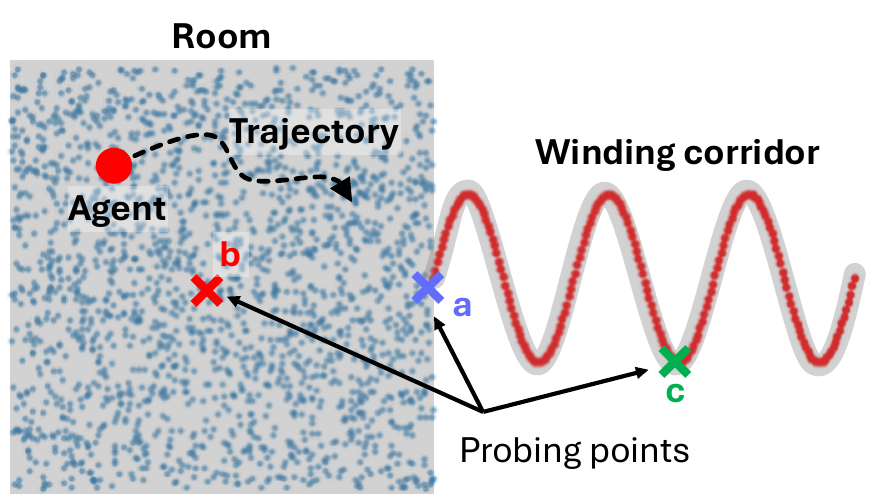}
    \caption{Example of an environment where an agent can move within a room or traverse a winding corridor. Dots are agent's position samples.}   
    \label{fig:room_corridor}
\end{figure}

Although the previous sub-sections have illustrated how the abstract machinery of poset stratifications can be useful for decomposing MDPs into different interesting regions, we now turn our attention to developing numerical methods for detecting stratifications automatically, especially in the latent representations of neural networks tasked to solve RL games.
Here we return to the classical theory of stratified spaces, where the defining feature are regions of varying dimensionality, cf.~Figure~\ref{fig:space-poset-dim}. 
To ground our numerical methods for later experiments, we first consider a ``noisy'' variation on Figure~\ref{fig:space-poset-dim}, shown in Figure~\ref{fig:room_corridor}, which we call the \emph{room with corridor} stratified space.
This example differs slightly from Figure~\ref{fig:space-poset-dim} because here the stratifying poset will have six distinct minimal elements: the four corners of the room and the two endpoints of the curvy corridor.
We note that each of these points will have slightly different volume growth laws, which we now define.

\begin{defn}
    Let $(X, d,\mu)$ be a metric measure space, where $\mu$ is the Borel measure on the space $X$ and $d:X \times X\to \mathbb{R}_{\geq 0}$ is a distance function. The \define{Volume Growth Transform (VGT)} associated to a point $x\in X$ is the log-log volume growth curve
    \[
        \mathrm{VGT}_x(s) \coloneq \log \mu(B_x(r=e^s))\,,
    \]
    where $B_x(r) = \{y \in X | d(x,y) < r\}$ is the ball of radius $r$.
\end{defn}

Although several techniques have been proposed for estimating local (intrinsic) dimensionality~\cite{mindml2012,twonn2017,fishers2019}, even under idealized uniform sampling conditions, some estimators tend to over-estimate the dimension in certain regions while others under-estimate it, as illustrated in Figure~\ref{fig:room_corridor_locdim}.
Consequently, although variations in the estimated local dimension are informative for detecting stratifications, the absolute values of those estimates should be treated carefully.

The bottom‑left panel of Figure ~\ref{fig:room_corridor_locdim} shows the local dimension estimated with the Volume Growth Transform (VGT)\footnote{The name Volume Growth Transform, or VGT for short, is not standard and was first proposed in~\cite{curry2025exploring}.}. First proposed in~\cite{gionis2005dimension} for clustering, the VGT was later used by~\cite{robinson2025token} to probe the latent geometry of large language models and subsequently adapted to the DRL setting in~\cite{curry2025exploring}.

\begin{figure}[t!]
    \centering
    \includegraphics[width=\columnwidth]{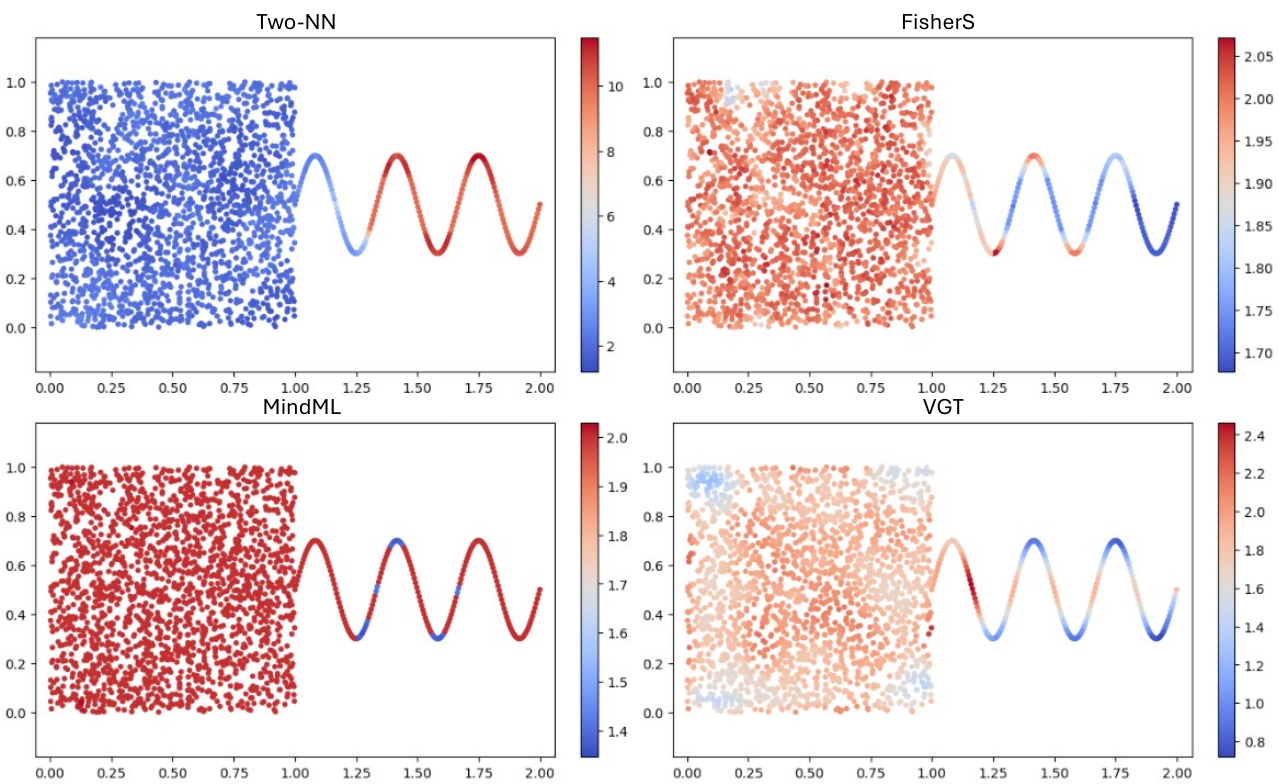}
    \caption{Local dimensions from various algorithms using \textsf{sklearn} library. Top-left: \cite{twonn2017}, top-right: \cite{fishers2019}, bottom-left: \cite{mindml2012} and bottom-right: \cite{curry2025exploring}.}
    \label{fig:room_corridor_locdim}
\end{figure}

Under small $r$ approximation and the assumption that the underlying space is a manifold, the logarithm of the volume $v_x$ of $B_x(r)$ scales linearly with the logarithm of the radius
\[
    \log v_x \approx K + n_x \log r\,,
\]
where $n_x$ is the local dimension at $x$. Thus estimates of $n_x$ can be computed by solving a linear least squares problem. 
Estimates for the environment in Figure~\ref{fig:room_corridor} are shown on the bottom-right of Figure~\ref{fig:room_corridor_locdim}.
If the space were a smooth manifold without boundaries, the estimated local dimension, $n_x$ would be (approximately) constant everywhere. 
In a stratified space such as Figure~\ref{fig:room_corridor}, the dimension should drop from 2 in the room to 1 in the corridor. The VGT captures this trend, although regions with high curvature (near the peaks of the sine wave) underestimate the dimensionality.

\begin{figure}[t!]
    \centering
    \includegraphics[width=0.8\columnwidth]{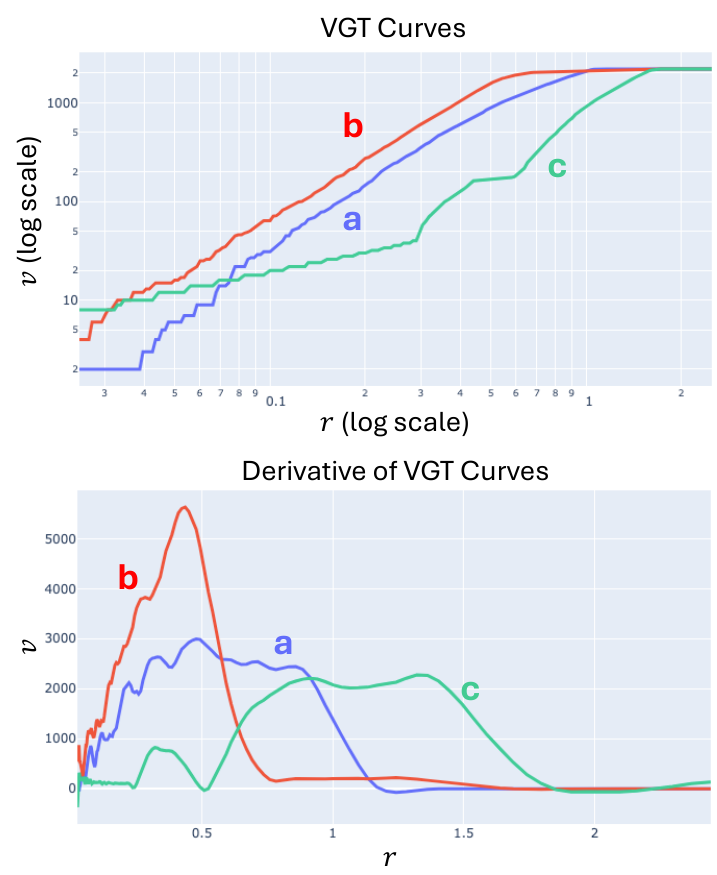}
    \caption{Top -- VGT curves for three points in the environment in Figure~\ref{fig:room_corridor}: connection point between room and corridor, point \textsf{a} (blue), center of the room, point \textsf{b} (red), middle minimum of the corridor point \textsf{c} (green). Bottom -- Smoothed derivative of the VGT curves. }
    \label{fig:VGT_curves}
\end{figure}

The VGT becomes especially informative when we examine the full scale of the curves, as shown in Figure~\ref{fig:VGT_curves}(top) for the three probe points marked in Figure~\ref{fig:room_corridor}. Both the blue and red curves have a slope of 2 (the blue curve, associated with \textsf{a}, is subtly different), while the green curve starts with a slope of 1. Around $r\approx0.3$ the green curve's slope increases as the ball $B_x(r)$ begins to capture volume from the adjacent room.
In this paper, we introduce the derivative of the (smoothed) VGT, \define{VGT‑dot}, as a feature for revealing the local‑to‑global stratifications found in a space.

\begin{rmk}
It is important to note that the green curve has a non-trivial piecewise linear behavior that follows the type of signatures characterized in~\cite[Theorem 1]{curry2025exploring} for stratified spaces. This provides further evidence of the stratified structure of the space.
\end{rmk}

Note that the VGT at the singular point \textsf{a} yields a slope near 2, illustrating how such critical points can be ``missed'' by local‑dimension estimates. Figure~\ref{fig:DIC_HADES_VGT-dot} compares several algorithms on (i) the room‑corridor layout of Fig.~\ref{fig:room_corridor} and (ii) an “hourglass'' space whose singularity sits at the junction of the two lobes, a configuration that is related to the DRL experiments in Section~\ref{sec:Stratification_DRL}. 
In particular, we compare~\cite{lim2025hades} (top), ~\cite{gionis2005dimension} (middle), and the proposed VGT-dot (bottom).
These observations will inform the analysis of the much higher-dimensional embedding space of a DRL agent.

There are several important considerations.  
First, HADES directly detects singularities and yields very accurate boundary and singular‑point identification for the two spaces shown in Figure~\ref{fig:DIC_HADES_VGT-dot}.  
Second, the DIC and VGT‑dot methods provide signatures via clustering.  DIC successfully clusters the room's boundary and the rim of the hourglass, but it misses the singularity at the hourglass's ``neck.''  
In contrast, VGT‑dot offers a more informative local‑to‑global clustering of the space.  
Finally, HADES localises singularities effectively but does not scale well to high‑dimensional data.  In contrast, DIC and VGT‑dot rely on linear regressions, so they scale readily, albeit at the expense of only coarse boundary‑ and singular‑point estimates.  

\begin{figure}[t!]
    \centering
    \includegraphics[width=0.90\columnwidth]{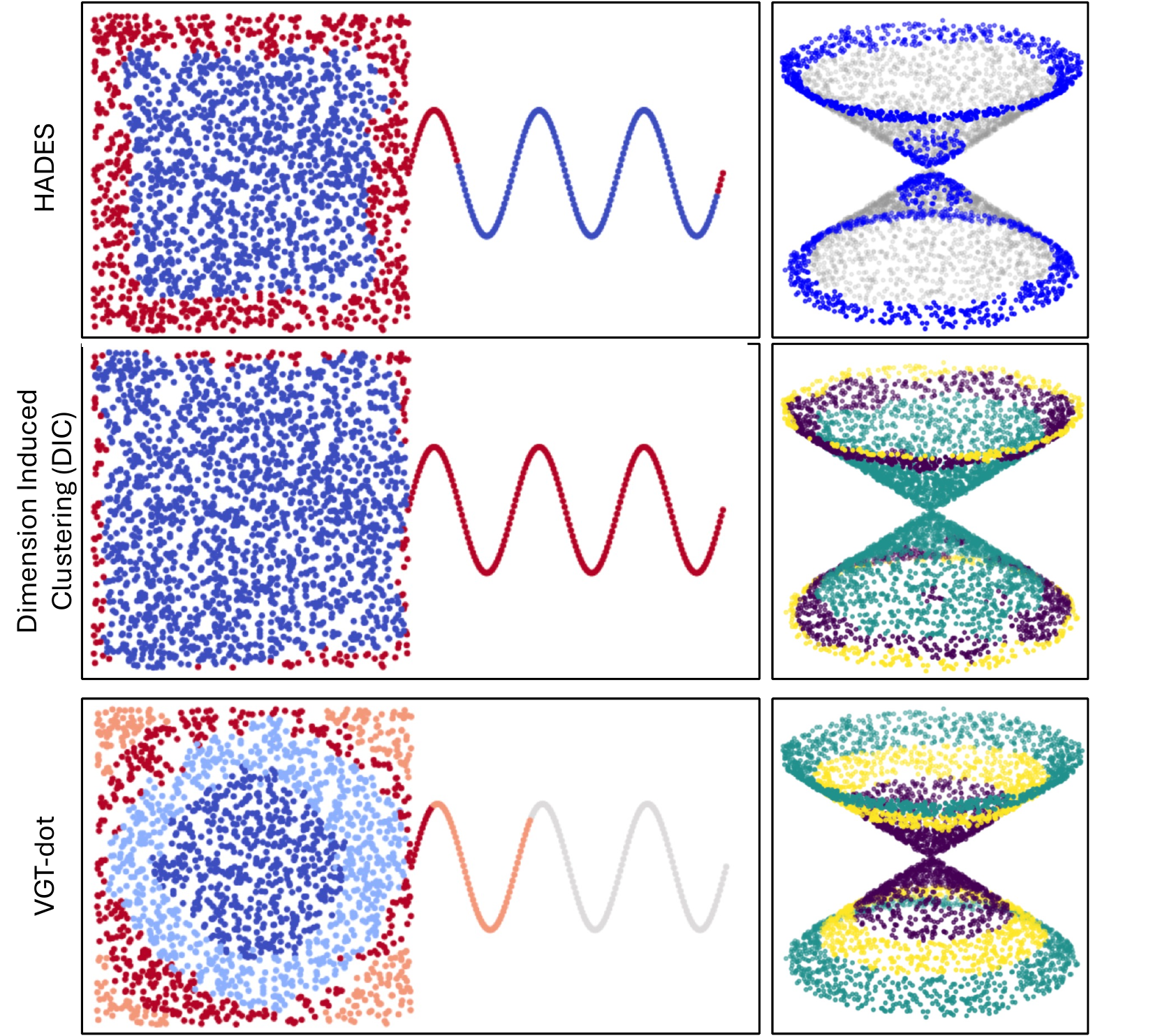}
    \caption{Stratification detection via HADES~\cite{lim2025hades} (top) and stratification signatures via Dimension Induced Clustering (middle) (DIC)~\cite{gionis2005dimension} and VGT-dot (bottom). The colorization of the middle and bottom figures is obtained by using Kmeans clustering on the DIC and VGT-dot feature, respectively.}
    \label{fig:DIC_HADES_VGT-dot}
\end{figure}

\section{Synthesizing Stratification Theory with STL}
\label{sec:STL_RL}

In this section, we expand on the discussion of volume growth laws from Section \ref{subsec:stratification_dectection} and suggest a practical implementation of the ``cartoon'' stratification for trajectory space suggested in Section \ref{subsec:strat_traj}. 
This then forms the theoretical backbone for extending Section \ref{subsec:digital-strats} to Section \ref{sec:Stratification_DRL}.

\subsection{Tubular Neighborhoods and their Growth Laws}
\label{subsec:tubular}

As Figure~\ref{fig:tube} illustrates, even a simple stratified space, like the unit interval embedded in the plane, can exhibit a mixed volume growth law, sometimes called a Steiner polynomial or Weyl Tube Formula\cite{katsnelson2009h}. 
In our example the volume of the tube of radius $r$ is given by $Vol(T_r):=2rL+\pi r^2.$

\begin{figure}[h]
    \centering
    \includegraphics[width=0.5\columnwidth]{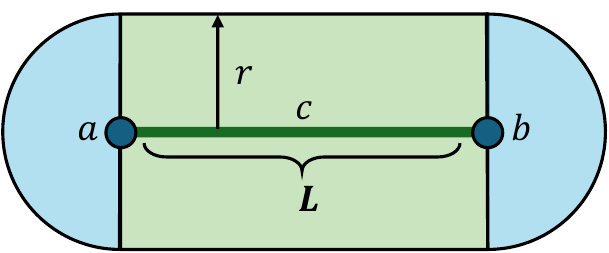}
    \caption{Tubular Neighborhood around a closed interval.}
    \label{fig:tube}
\end{figure}

Interestingly, the lower dimensional strata $a$ and $b$ exhibit a faster growth law, as their volume grows according to the \emph{codimension} inside the ambient space.
This suggests that perhaps one could design a numerical method for isolating strata using the VGT after first eliminating lower-dimensional sub-strata.
First, we formalize the existence of tubular neighborhoods in the following lemma.

\begin{lem}[Existence of Tubular Neighborhoods\cite{durfee1983}]\label{lem:tubes-exist}
    Suppose $S$ is a compact semi-algebraic subset of $\R^n$, i.e.,
    it is cut out by polynomial inequalities.
    Such a set $S$ admits an algebraic tubular neighborhood $T_r\supset S$, i.e., there is a semi-algebraic function $g: \R^n \to \R$ such that $g^{-1}(0)=S$ and $g^{-1}[0,r]=T$ and $T$ deformation retracts onto $S$.
\end{lem}

Lemma \ref{lem:tubes-exist} was first proved by \cite{durfee1983} by generalizing Milnor's ``rug function'' technique from algebraic varieties to the compact semi-algebraic setting.
Later work \cite{dutertre2009} generalized this result further to the closed, but possibly non-compact setting. 
A key ingredient in the statement of \cite[Lem.~1.4]{durfee1983} is that the defining rug function $g:\R^n\to\R$ for $T$ has finitely many critical values. As long as $r<\delta_1$, where $\delta_1$ is the first critical value of $g$, then $T_r$ is a good tubular neighborhood for $S$.
For concreteness, when $S$ is a smooth submanifold of $\R^n$, one can take
$g_S(x)$ to be the squared distance from $S$, i.e., $d(S,x)^2$.
The first critical value $\delta_1$ is the \define{reach} of the subspace $S$, which is finite in the following example.

\begin{ex}\label{ex:circle}
    Consider $S=\mathbb{S}^1\subseteq \R^2$ the unit circle embedded in the plane. The first, and only, critical value for $g_S$ is at $\delta_1=1$, where the thickened circle eventually includes the origin. This example illustrates how curvature in the submanifold leads to finite reach.
\end{ex}

\subsection{STL Basics}
\label{subsec:STL-basics}

\define{Signal Temporal Logic (STL)} \cite{maler2004monitoring} is a logic adapted to the study of continuous, time-varying signals.
For us, these signals arise from studying trajectories $\gamma:[0,T]\to X\subseteq \R^n$ and functions $f:\R^n\to\R$.
Atomic predicates, $\alpha_f$, for STL are statements of the form $ f(x)\geq \mu$ or $f(x)\leq \mu$ for some continuous function $f:X\to \R$.
One of the key features of STL is that it uses the value $|f(x)-\mu|$ to assign a \define{robustness score} to $\alpha_f$ at the point $x$.
Integrating this over a trajectory $\gamma$ then assigns an \define{STL score} to any STL formula.
Formally, the robustness function $\rho$ is defined recursively by first specifying values on atomic predicates.
Formally, $\rho(\gamma,\alpha_f,t)=f(\gamma(t))-\mu$ if $\alpha_f$ is the atomic predicate $f(x)\geq \mu$ and $\rho(\gamma,\alpha_f,t)=\mu -f(\gamma(t))$ if $\alpha_f$ is the atomic predicate $f(x)\leq \mu$.
From there one then uses minus signs, min, and/or max to specify the robustness value associated to $\neg$ (logical negation), $\wedge$ (logical AND), or $\vee$ (logical OR).
The modal \textbf{until} operator $\Until_{I}$, written $\phi_1 \Until_{I} \phi_2$, is true at time $t$, if there exists some $t'\in t + I$ such that $\phi_2$ is true at time $t'$ and $\phi_1$ is true \emph{for all} times $t''\in [t,t']$. 
From until we can then define \define{eventually true} $\Event_{I} \phi \equiv \True \Until_{I} \phi$, which is true at time $t$ if $\exists t'\in t + I$ where $\phi$ is true at $t'$.
Finally, one can define \define{always true}
$\Always_{I} \phi \equiv \notltl \Event_{I} \notltl \phi$, which is true if $\phi$ is true throughout the next interval $I$.

\subsection{An STL Formulation for Stratified Spaces}
\label{subsec:STL-for-Strat}


We now demonstrate how to theoretically define atomic predicates associated to strata in a semi-algebraically stratified space.
These predicates essentially state that we are within some maximal tubular neighborhood of a stratum.
We conclude this section by showing how one can combine these predicates to specify formulae for deciding membership in an abstractly defined stratified space.

\begin{defn}
    Fix a bounding box $\mathbb{B}^n\subset \R^n$ with non-empty interior and a compact semi-algebraic set $S\subseteq \mathbb{B}$ of dimension less than $n$. 
    Let $T_{max}$ be a maximal semi-algebraic neighborhood of $S$, which is further contained in $\mathbb{B}^n$.
    Let $\mu_S=d(S,\partial T_{max})$, which is either the reach of $S$ or the minimum distance from $S$ to the boundary of the box $\partial \mathbb{B}^n$.
    In either case, define
    $\alpha_S \colon d(S,x) \leq \mu_S$
    to be the \define{close to $S$ predicate}.
    This is the maximally robust predicate for determining membership in the subspace $S$.
    Dually, define the \define{away from $S$ predicate} to be
    $\beta_S \colon d(S,x)\geq 0$.
\end{defn}

\begin{thm}\label{thm:strat-decision}
    Suppose $X\subseteq \mathbb{B}^n$ is a semi-algebraic subset, stratified into semi-algebraic subsets $\{S_i\}_{i\leq P}$ satisfying the axiom of the frontier, i.e., that $S_i\cap \overline{S_j}\neq \varnothing$ implies $S_i\subseteq \overline{S_j}$, which defines the indexing poset $P$.
    The areas of satisfaction for each close-to-$S_i$ predicate $\alpha_i$ form a poset, isomorphic to $P^{op}$.
    Moreover, we can decide membership in a fixed stratum $S_i$ by evaluating $\alpha_i \wedge \left(\wedge_{j< i} \beta_{S_j} \right)$ assuming the frontier strata $S_j< S_i$ are known.
\end{thm}

\section{Stratification of a DRL Game}
\label{sec:Stratification_DRL}


\subsection{STL-Reward Based Game}

We study a Minigrid game whose state is $s=(x,y,\theta,t)\in[0,84]\times[0,84]\times\{0,\pm45,90\}\times[0,194]$, encoding position, heading, and time.  
The discrete action set consists of the four cardinal, four inter-cardinal directions, and a ``no‑action'' command.  
Using the MemoryGym environment~\cite{pleines2023memory}, we define two games in which the reward is the normalized robustness (in $[-1,1]$) of an STL formula, Figure~\ref{fig:stl_game_env}.  
Because STL robustness can only be evaluated on a complete trajectory, the agent receives a delayed reward at the end of each episode ($t=195$).

\begin{figure}[t!]
    \centering
    \includegraphics[width=0.7\columnwidth]{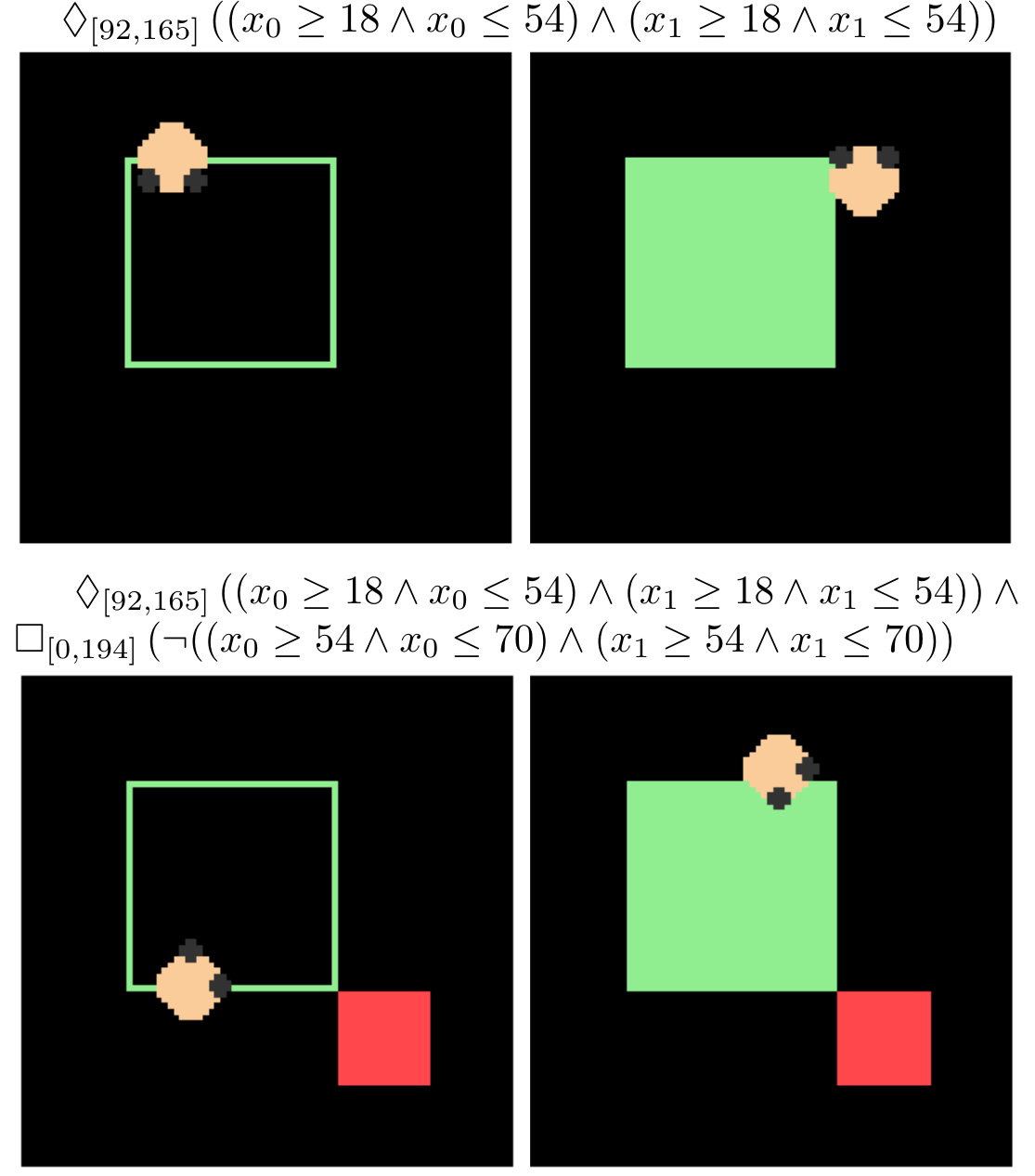}
    \caption{Screenshots of the environment and the observation used as an input to the transformer network. The light orange avatar represents the agent. The left images show the empty square occurring before or after the ``eventually'' operator is active and the right images with the filled square showing the instances when the temporal operator is active. The red square on the bottom is an area that must be avoided to get high reward.}
    \label{fig:stl_game_env}
\end{figure}

Our policy is learned via a Transformer‑XL model trained using PPO as described in~\cite{memory_gym_github}.  
Each token comprises the current observation and the 10 most recent frames, two shown in Figure~\ref{fig:stl_game_env}.  
A pretrained and fixed image‑processing backbone extracts visual features, which are passed through a learnable token‑embedding layer, two transformer blocks, and a fully‑connected head that maps the latent representation to an action via soft-max.  
Time is not supplied explicitly to the agent, but a square in the image/observation remains empty until the temporal operator becomes active, when it is filled.  
The left images of Figure~\ref{fig:stl_game_env} show observation before (and after) the operator $\lozenge_{[92,165]}$ is active, and the right images depict the environment once the operator is active.

We examine the stratification of the token embedding space corresponding to the space where the transformer maps visual features to a useful representation.  
We focus on this initial layer because early embeddings largely govern the network's subsequent actions. Unlike standard vision transformers that attend to localized image patches, each token here corresponds to the full image.  
In other words, each point in Figure~\ref{fig:UMAP_VT-dot_eventually} represents an entire screenshot, i.e., a state of the MDP. 
Similar to language models, the transformer core can be viewed as a next-image predictor, with a final fully-connected head that translates that prediction into a reward-maximizing action. 

\subsection{Stratification of STL-Based DRL}

After training the DRL agent to satisfactory performance, we sample 250 trajectories from random initial states.  Each trajectory consists of 194 steps, yielding $250\times 194 = 48,500$ token embeddings of dimension 256.  The set of unique embeddings is much smaller: roughly 7.6k distinct vectors for the “eventually” game and about 12k for the ``eventually + always not'' game.  To probe the stratification of this token embedding space we apply the techniques described in Section~\ref{subsec:stratification_dectection}.  The embeddings are projected into 3D with UMAP~\cite{mcinnes2018umap} or ISOMAP~\cite{tenenbaum2000global}, and the resulting visualizations highlight the features and metrics of interest.

\subsubsection{``Eventually in the green square''}

\begin{figure}[t!]
    \centering
    \includegraphics[width=0.90\columnwidth]{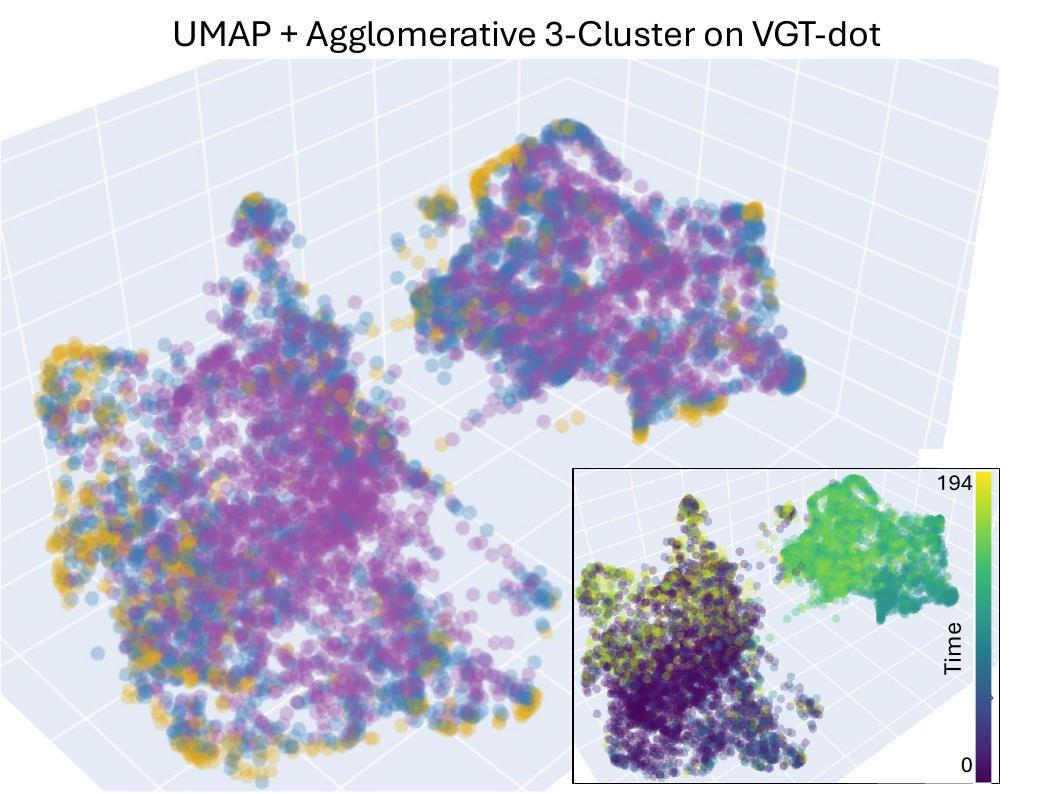}
    \caption{Main panel: UMAP embedding of the token vectors, colored by the three clusters obtained with agglomerative clustering using VGT‑dot as the feature.  The colors trace a pattern that resembles the ``hourglass'' structure shown at the bottom of Figure~\ref{fig:DIC_HADES_VGT-dot} (the correspondence is slightly blurred by embedding noise).  Inset: the same UMAP projection, now colored by time (0–194).  The right‑hand point cloud corresponds to timestamps when the ``eventually'' operator is active, when the green square is filled. }
    \label{fig:UMAP_VT-dot_eventually}
\end{figure}

The UMAP projection of the token space, Figure~\ref{fig:UMAP_VT-dot_eventually}, reveals two clearly separated point clouds.  
The left cloud contains tokens from observations in which the square is empty, while the right cloud groups tokens from observations where the square is filled.  An inset zooms in on this separation by coloring points according to time: dark and yellow hues correspond to the periods \emph{before} and \emph{after} the square becomes filled, and a lighter‑green hue marks the tokens generated while the square is filled.  
The existence of two distinct clouds is not surprising, since the presence or absence of the square provides a strong discriminative cue.  Moreover, because the transition from empty to filled to empty follows the temporal evolution of the episode, the UMAP embedding naturally captures an hourglass-like shape akin to the one in Section~\ref{subsec:stratification_dectection} (plotted here horizontally). 
The colors used for the large UMAP embedding is the agglomerative clustering (for three clusters) of the VGT-dot of the token space. While the embedding is noisier than the hourglass example given that samples are sparser, we still see symmetric hues about the ``neck'' going from purple to blue to yellow as we move from the neck to the periphery of the point clouds. 

To apply HADES to this point cloud we first reduced the 256‑dimensional token embeddings to 100 dimensions using a discrete cosine transform (DCT) matrix, which serves as a random projection that preserves the underlying geometry~\cite{bingham2001random}.  
The resulting 3D UMAP projection of the 100D space in Figure~\ref{fig:UMAP-HADES} retains the main structure observed in Figure~\ref{fig:UMAP_VT-dot_eventually}, except the point cloud is now shown vertically.

Points colored purple in Figure~\ref{fig:UMAP-HADES} are those that HADES flags as failing the manifold test with some statistical certainty; see \cite{lim2025hades} for more details.  
These points lie at the neck of the hourglass, matching the VGT‑dot pattern displayed in Figure~\ref{fig:DIC_HADES_VGT-dot}.  
The correspondence between the VGT‑dot signature in the original 256D token space and the HADES classification in the reduced 100D space provide strong evidence of stratifications within both token representations. 
HADES also identifies non-manifold points in the upper portion of the top point cloud, Figure~\ref{fig:UMAP-HADES}.  
These points tend to belong to the later stages of trajectories, possibly indicating boundary points.
Recall that a manifold with boundary is technically not a manifold, as boundary points have a lower dimension.

\begin{figure}[t!]
    \centering
    \includegraphics[width=0.63\columnwidth]{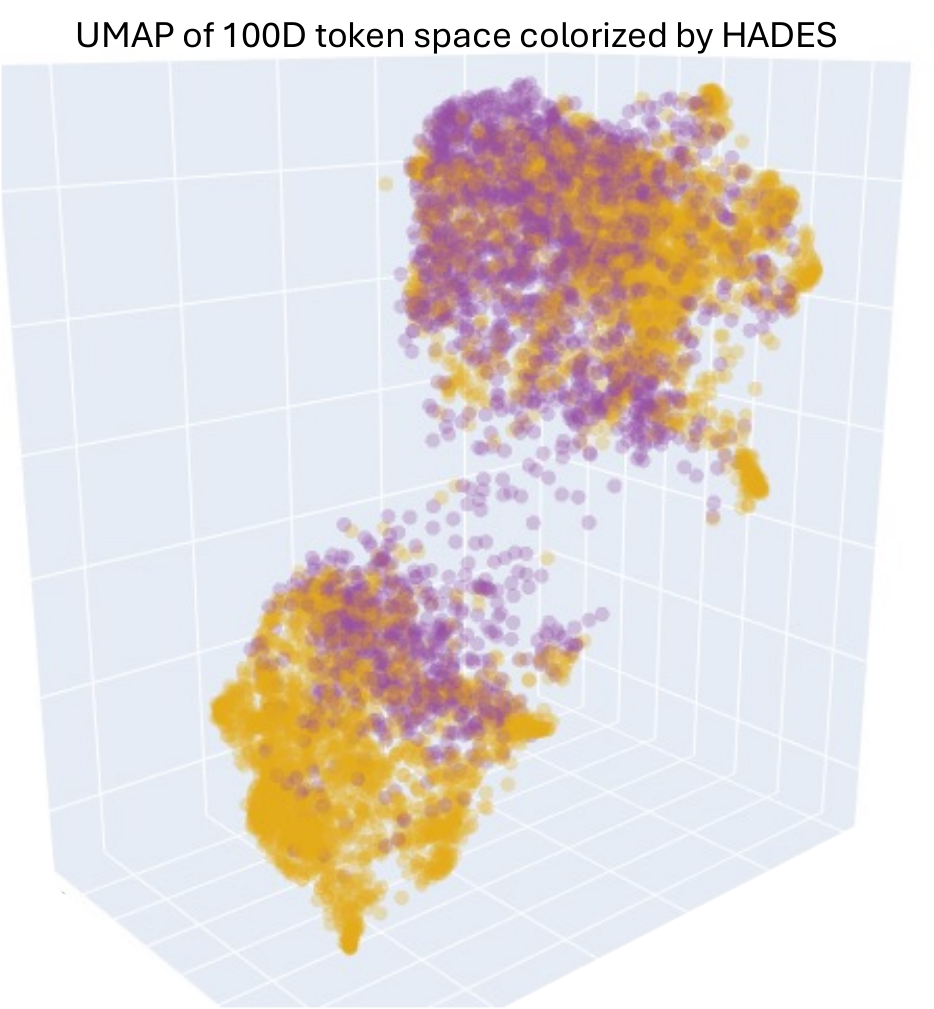}
    \caption{UMAP projection to 3D from a 100D DCT projection of the original 256D token space. Note that the point cloud retains the geometric ``hourglass'' structure from Figure~\ref{fig:UMAP_VT-dot_eventually}, whose idealized form was studied in the right column of Figure~\ref{fig:DIC_HADES_VGT-dot}. Purple points are those failing the manifold hypothesis, rejected with statistical certainty by HADES. Note how they concentrate near the ``choke point'' in a noisy ``hourglass.''}
    \label{fig:UMAP-HADES}
\end{figure}

\subsubsection{``Eventually in the green square and always not in the red square''}

Projecting our 256D token space to 3D again yields two distinct clusters, Figure~\ref{fig:umap_isomap_vgt-dot_always} (left), mirroring the pattern observed in the previous ``eventually'' STL game.  
Coloring the UMAP embedding by VGT-dot cluster values reveals the same structure as before: an inner magenta core cluster expands outward to a yellow boundary.  
Because the number of tokens is substantially larger, the HADES algorithm exceeds available memory, even after applying DCT to 100D; so we omit HADES for this dataset. 
Instead, we present the VGT‑dot analysis on the ISOMAP embedding, Figure~\ref{fig:umap_isomap_vgt-dot_always} (right).  
The dominant visual cue, the filled versus empty square, produces two well‑defined lobes as before.  
Clustering based on the VGT‑dot feature again exhibits the symmetric hourglass-like pattern identified previously.  
Since we do not have HADES providing statistical certainty for non-manifold structure, our claim of a stratified space structure here is weaker.
However, the evident ``hourglass'', viewed across multiple methods, coheres with our interpretation that the agent must navigate through a distinguished pinch point in order to successfully win the game.

\begin{figure}[t!]
    \centering
    \includegraphics[width=\columnwidth]{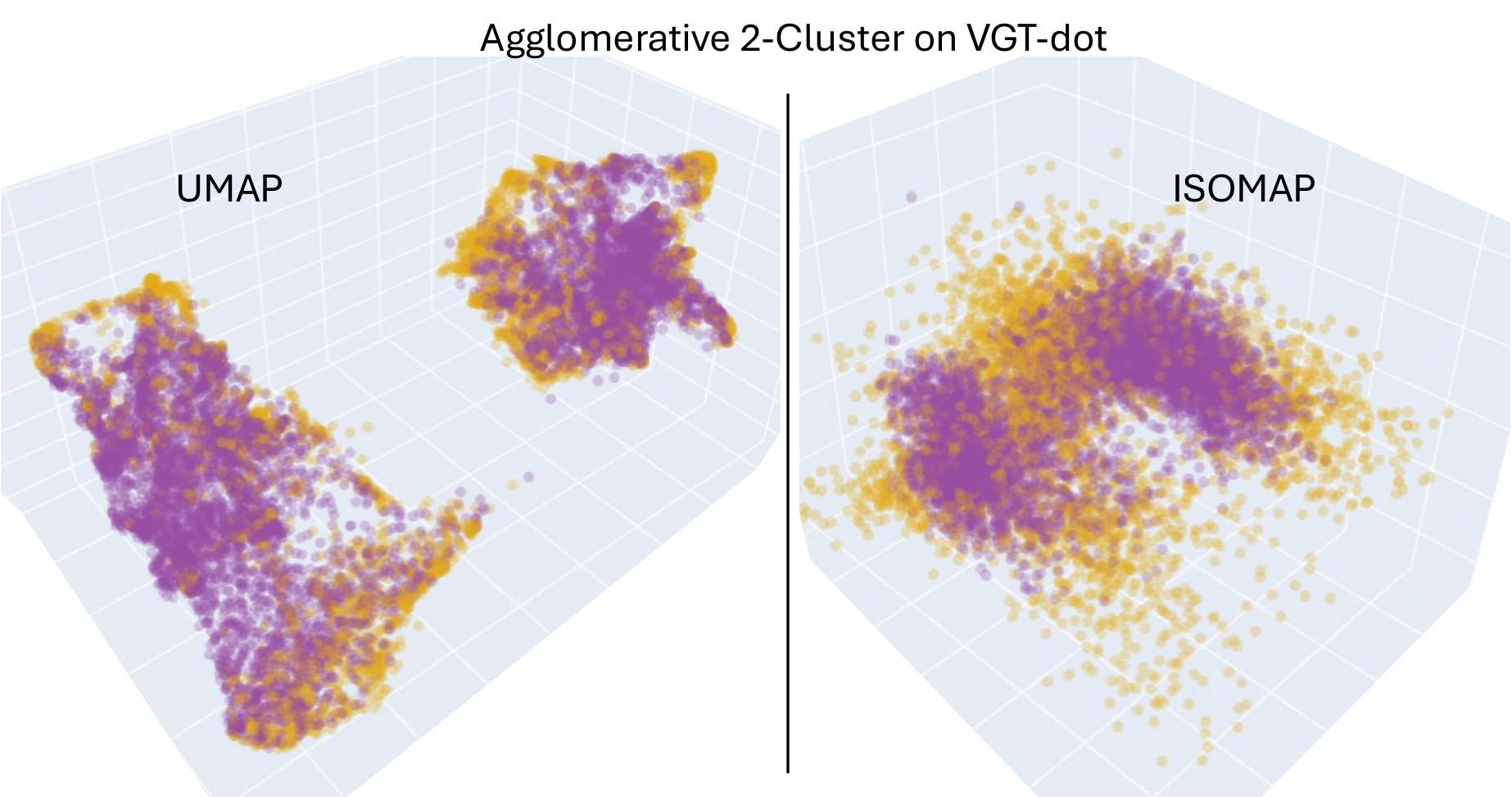}
    \caption{Left: UMAP embedding of the token space.  Right: ISOMAP embedding of the token space. Both projections are colored by agglomerative clustering (two clusters) using the VGT‑dot as the feature.  The resulting pattern mirrors that of the ``eventually'' game and the hourglass space, with a symmetric core cluster expanding into an outer cluster.}
    \label{fig:umap_isomap_vgt-dot_always}
\end{figure}

\section{Conclusions}

In this paper, we have established a novel poset-based stratification theory for modeling Minigrid games, along with their latent (transformer-based) representations.  
We introduced a new class of atomic predicates for STL formulae, treating them as defining strata in the underlying space–time manifold and used robustness as a reward function for training a DRL agent.   
Preliminary numerical experiments on 256D token spaces, using a new VGT-dot technique, support our stratification hypothesis and coheres with other methods, such as DIC and HADES.
However, additional research is required to develop more robust tubular‑neighborhood based volume growth laws to automate extraction of strata.




\section{Acknowledgements}
We would like to thank Ngoc B. Lam at Lockheed Martin for setting up the simulation environment with STL-based reward, and  Brennan Lagasse who, during his intership at Lockheed Martin, adapted the environment from MemoryGym and developed the scripts to save token embeddings.

\bibliographystyle{ieeetr}
\bibliography{references}

\end{document}